\documentclass[runningheads]{llncs}
\usepackage{graphicx}

\usepackage{tikz}
\usepackage{comment}
\usepackage{color}
\usepackage{threeparttable}
\usepackage[accsupp]{axessibility}  

\usepackage[hidelinks]{hyperref}
\usepackage[utf8]{inputenc}
\usepackage[small]{caption}
\usepackage{booktabs}
\usepackage{algorithm}
\usepackage{algorithmic}
\usepackage[mathscr]{eucal}
\usepackage{bm}
\usepackage{amssymb}
\usepackage{multirow}
\usepackage{hyperref}
\hypersetup{hidelinks,
	colorlinks=true,
	citecolor=green,
	urlcolor=magenta, 
	linkcolor=blue,
	pdfstartview=Fit,
	breaklinks=true}

\begin{document}
\pagestyle{headings}
\mainmatter
\def\ECCVSubNumber{8003}  

\title{Kernel Relative-prototype Spectral Filtering for Few-shot Learning} 


\titlerunning{Kernel Relative-prototype Spectral Filtering for Few-shot Learning}
%
\author{Tao Zhang\inst{1}\orcidID{0000-0002-0281-9234} \and
         Wu Huang\inst{2}\orcidID{0000-0002-2525-6454}}
\authorrunning{T. Zhang et al.}
%
\institute{Chengdu Techman Software Co., Ltd., Chengdu,
Sichuan, China
\email{ztuestc@outlook.com}\\ \and
Sichuan University, Chengdu, Sichuan, China\\
\email{huangwu@scu.edu.cn}}
\maketitle

\begin{abstract}
	Few-shot learning performs classification tasks and regression tasks on scarce samples. As one of the most representative few-shot learning models, Prototypical Network represents each class as sample average, or a prototype, and measures the similarity of samples and prototypes by Euclidean distance. In this paper, we propose a framework of spectral filtering (shrinkage) for measuring the difference between query samples and prototypes, or namely the relative prototypes, in a reproducing kernel Hilbert space (RKHS). In this framework, we further propose a method utilizing Tikhonov regularization as the filter function for few-shot classification. We conduct several experiments to verify our method utilizing different kernels based on the \emph{mini}ImageNet dataset, \emph{tiered}-ImageNet dataset and CIFAR-FS dataset. The experimental results show that the proposed model can perform the state-of-the-art. In addition, the experimental results show that the proposed shrinkage method can boost the performance. Source code is available at \url{https://github.com/zhangtao2022/DSFN}.
	\keywords{Few-shot learning, Relative-Prototype, Spectral Filtering, Shrinkage, Kernel}
\end{abstract}

\section{Introduction}

Humans have an innate ability to quickly learn from one or several labeled pictures and infer the category of new pictures.
In contrast, deep learning, despite its breakthrough success in computer vision, still needs huge data to drive it. This shortcoming seriously hinders its applications in some practical situations where data is scanty. Therefore, it is an important and challenging problem for machines to acquire the human-like ability to make inferences about unknown samples based on too few samples. Inspired by this ability of humans, few-shot learning is proposed and has become a hot spot \cite{feifei2006oneshot,vanschoren2018meta-learning,2015Human}. Vinyals et al. \cite{vinyals2016matching} proposed a training paradigm that few-shot learning models should learn new categories of unseen examples from query set using very few examples from support set. Meta-learning methods can be well applied to the few-shot settings that need to complete the task that contains both support set and query set.

\begin{figure}[h]
	\centering
	{\includegraphics[width=0.9\textwidth]{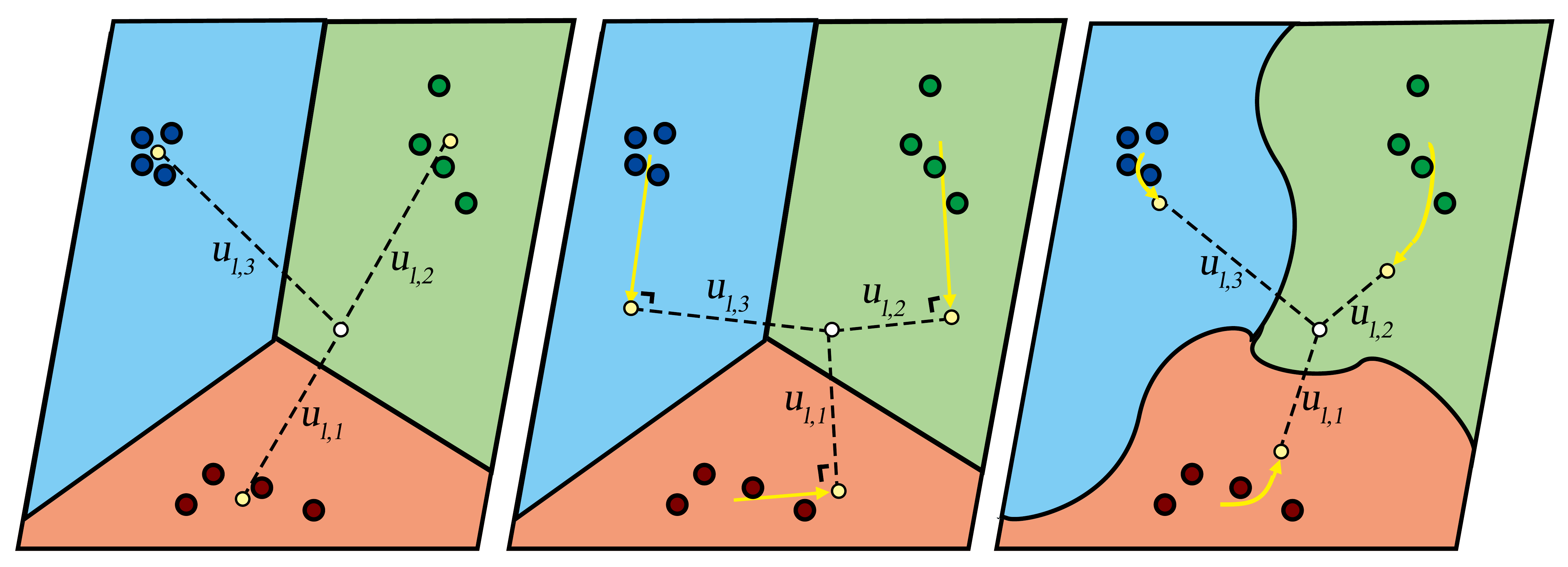}}
	\caption[]{Comparison of the relative-prototypes (dotted line) in the Prototypical Networks (ProtoNet)\cite{snell2017prototypical}, Deep Subspace Networks (DSN)\cite{simon2020adaptive} and the proposed DSFN. \textbf{Left}: ProtoNet in Euclidean space; \textbf{Middle}: DSN in Euclidean space; \textbf{Right}: DSFN in a reproducing kernel Hilbert space. For them, $u_{l,1}$,$u_{l,2}$ and $u_{l,3}$ are the relative-prototypes with class 1(brown), class 2(green) and class 3(blue), respectively.}
	\label{Proposed}
\end{figure}

Recently, a series of meta-learning models for few-shot setting have been proposed, which can be divided into the metric-based models and the optimization-based models \cite{snell2017prototypical,finn2017model-agnostic,2018Learning,munkhdalai2017meta,lee2019meta-learning,mishra2018a,andrychowicz2016learning,ravi2016optimization,koch2015siamese,zhang2018zero,gidaris2018dynamic,mangla2020charting,rodriguez2020embedding,ziko2020laplacian}. Prototypical Networks  (ProtoNet), as one of the metric-based representatives, proposes to use prototypes to represent each category, and to measure the similarity between a sample and a prototype by using Euclidean distance \cite{snell2017prototypical}. Based on this idea, some prototype-related models have been proposed in recent years \cite{chen2020variational,oreshkin2018tadam:,fort2017gaussian,pahde2021multimodal}.

The prototype in ProtoNet is estimated by support sample mean in each class, which may deviate from the true prototype \cite{liu2020prototype}. In \cite{liu2020prototype},  Two bias elimination mechanisms are proposed to eliminate the difference between the prototype estimated value and the true value. In addition, a method of using a combination of semantic information and visual information to better estimate the prototype was proposed, and a significant improvement was obtained  \cite{xing2019adaptive}. 

In measurement of sample similarity, Mahalanobis distance would be better than Euclidean distance  to capture the information from the distribution within the class\cite{bateni2020improved}. In addition, similarity measurement in hyperbolic space that learns the features of a hierarchical structure could be better than those in Euclidean space for few-shot image classification \cite{fang2021kernel,khrulkov2020hyperbolic}.

In this paper, we propose a method called Deep Spectral Filtering Networks (DSFN) aiming to better estimate relative prototypes, or the difference between prototypes and query samples (Fig. \ref{Proposed}). Euclidean distance can not fully capture the information of intra-class difference as it is applied to measure the difference between sample and prototype. Thus, the main components of the inner class distance are taken into account in the similarity assessment between the sample and the prototype, which to some extent interferes with this assessment. In our approach, the influence of these components is weakened via spectral filtering. It is similar to the kernel mean shrinkage estimation that aims to reduce the expected risk of mean estimation \cite{muandet2016kernel1,muandet2014kernel1,muandet2014kernel,muandet2016kernel}, but the estimated relative prototypes don't have to be close to the true mean. 

\begin{figure}[h]
	\centering
	{\includegraphics[width=0.95\textwidth]{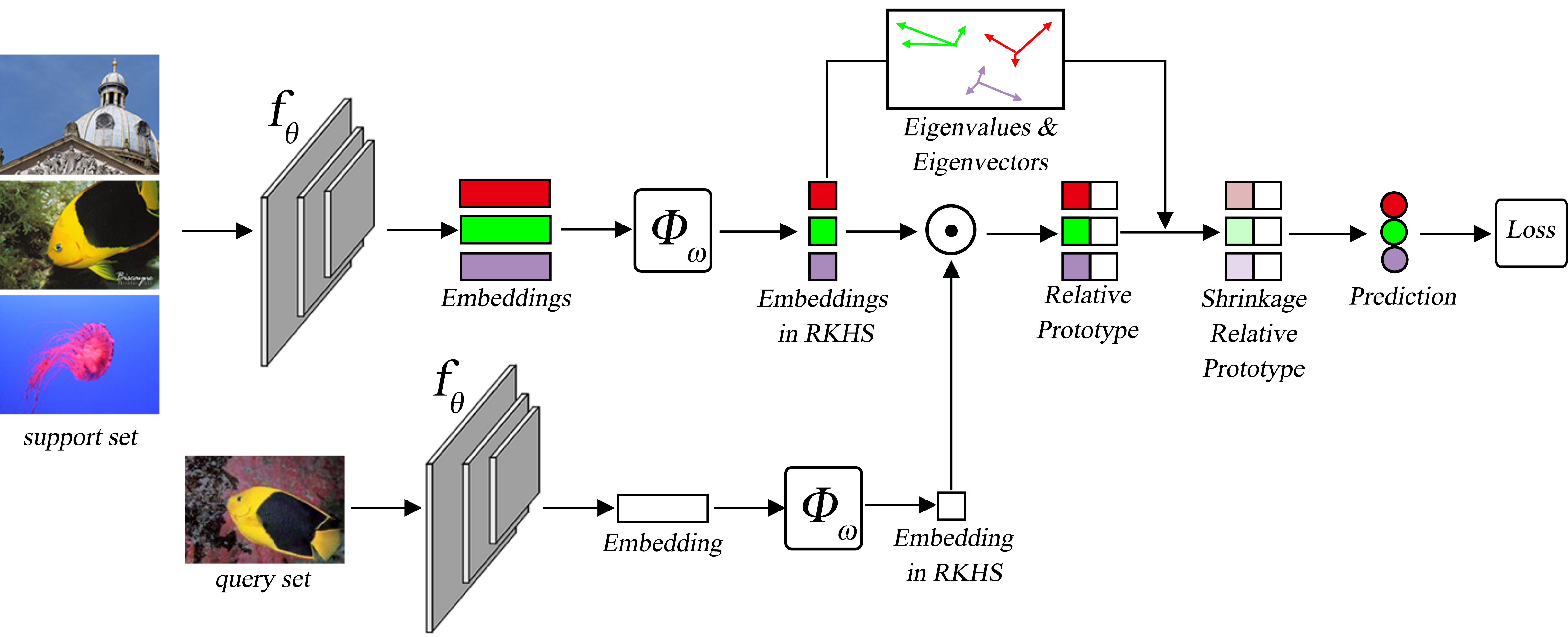}}
	\caption[]{The overview of our proposed approach. The features of support set and query set extracted from $f_\theta$ are mapped into RKHS by the function $\phi_\omega$. The relative prototypes are shrunk based on the eigenvalues and eigenvectors from the support set.}
	\label{overview}
\end{figure}

Recent work has shown that kernel embedding can significantly improve the performance of few-shot learning \cite{fang2021kernel}. Based on kernel mean embedding theory, our approach can measure the relative prototypes in a reproducing kernel Hilbert space (RKHS). The advantage of this is that we can deal with the projection problem of feature space by means of the kernel embedding. The overview of our proposed approach is shown in Fig. \ref{overview}.

The contributions of our work are summarized as four folds:
\begin{itemize}
	\item[1)] To our knowledge, this work is the first to estimate relative-prototypes using a kernel shrinkage method for few-shot learning.
	\item[2)] We propose to estimate relative prototypes instead of prototypes, aiming to better capture the probability information that the query samples belongs to a class.
	\item[3)] We propose a framework of spectral filtering to estimate the relative prototypes. This approach can filter the interference within cluster variation while measuring the relative prototypes and boost the performance.
	\item[4)] We introduce kernel embedding for measuring the relative prototypes via spectral filtering in RKHS, which allows us to apply the kernels that achieve the state-of-the-art performance.
\end{itemize}

\section{Related Work}
\subsection{Metric-based Few-shot Learning}
Recently, some metric-based meta-learning algorithms for few-shot setting, one-shot setting and zero-shot setting were proposed, and the representatives of metric-based meta learning are  Matching Networks (MatchingNet) \cite{vinyals2016matching} and ProtoNet \cite{snell2017prototypical}. MatchingNet proposed an attention mechanism for rapid learning. In addition, it pointed out that in the training procedure, the test condition should match with the train condition. ProtoNet proposed a simple strategy to improve the performance of MatchingNet. The strategy is using prototype, the mean of support samples in each class, to represent its support set\cite{snell2017prototypical}. From the perspective of classifier design, some works proposed a different kind of classifier with MatchingNet and ProtoNet. Simon et al. proposed the dynamic classifiers by using subspaces for few-shot learning called Deep Subspace Networks (DSN)\cite{simon2020adaptive}. These classifiers are defined as the projection of the difference between a prototype and a query sample onto a subspace. Lee et al. \cite{lee2019meta-learning} proposed a linear support vector machine instead of nearest neighbor classifiers for few-shot learning.

Our work is related to the rectification of the prototype. For ProtoNet, the prototype is simply calculated by averaging the support sample values. Some kinds of work show that, for improving the classification performance, the prototype need to be better estimated. For example, compared with the simple visual information, the combination of cross-modal information and visual information can better represent the prototype \cite{xing2019adaptive,pahde2021multimodal}. In addition, a method of transductive setting has been proposed to rectify prototype, which diminishes the intra-class bias via label propagation and diminishes the cross-class bias via feature shifting \cite{liu2020prototype}.  Gao et al. used a combination of the instance-level attention and the feature-level attention for the noisy few-shot relation classification \cite{gao2019hybrid}. In this method, prototype is a weighted mean of support samples after conducting by the instance-level attention.

Our work is also related to the similarity measurement of two samples. Similarity measurement between samples in ProtoNet is carried out by using Euclidean distance. Euclidean distance implies two assumptions, that is, the characteristic dimensions are not correlated and have consistent covariance. However, in the real world these two assumptions are not necessarily true.  Boris N. Oreshkin et al. found that, by using the metric scaling, the performance of few-shot algorithms can be improved by optimizing the metric scaling, and showed that the performance using scaled cosine distance nears that using Euclidean distance\cite{oreshkin2018tadam:}. The metric scaling can also be learned from a Bayesian method perspective.\cite{chen2020variational}. Sung et al. proposed the Relation Networks (Relation Nets) that learns to learn a transferrable deep metric\cite{2018Learning}. In addition, Mahalanobis distance is proposed to overcome the defect of Euclidean distance on measuring the similarity between two samples \cite{bateni2020improved,fort2017gaussian}. Recently, Khrulkov et al. studied how to measure sample similarity in a hyperbolic space instead of Euclidean space. They found that, compared with Euclidean space,  hyperbolic embeddings can benefit the embedding of images and provide a better performance \cite{khrulkov2020hyperbolic}. However, it is more difficult to operate in hyperbolic space, such as sample averaging, which further hinders the use of hyperbolic geometry. Fang et al. \cite{fang2021kernel} provides several positive definite kernel functions of hyperbolic spaces, which enable one to operate in hyperbolic spaces. 

\subsection{Kernel Mean Shrinkage Estimation}
The method of kernel mean shrinkage estimation is relevant to our work, which is employed to estimate the relative prototype. In recent years, some work has been done on the kernel mean shrinkage estimation. Muandet  et al. \cite{muandet2014kernel1} pointed out that the estimation of an empirical average in RKHS can be improved by employing Stein effect. Like James-Stein estimator, they propose several kernel mean shrinkage estimators, e. g.,  empirical-bound kernel mean shrinkage estimator (B-KMSE) and spectral kernel mean shrinkage estimator (S-KMSE) \cite{muandet2016kernel}.  B-KMSE is easily optimized but the degree of shrinkage is the same for all coordinates.  S-KMSE can shrink differently on each coordinate on the basis of eigenspectrum of covariance matrix, but it is difficult to optimize. Furthermore, ``Shrinkage'' can be achieved by a generalized filter function that can be embodied by various forms, such as Truncated SVD \cite{muandet2014kernel}. For these estimators,  the shrinkage parameters are commonly obtained by a cross-validation score.

\section{Methodology}

\subsection{Preliminary}
In metric-based few-shot learning, before measuring the difference between query and classes, one often firstly represent each class by using support samples belonging to them, respectively. Here we consider two different geometric types of the class representation.

\textbf{Point type.} In this type, a class is often represented as a point, e.g., a prototype, which is calculated by averaging of support samples in each class \cite{snell2017prototypical}. In addition, the point can also be calculated by summing over the support samples\cite{2018Learning}. In these cases, The probability of measuring a sample belonging to a class is usually based on the distance between two points. For example, for ProtoNets, the distance square is expressed as
\begin{equation}\label{ProtoNet_m}
d_{l,c}^2 = \Vert f_\theta\left(\bm{q}_l\right) - \bm{\mu}_{c}\Vert^2,\,\,\, {\rm{with}}\,\,\,\bm{\mu}_{c} = \frac{1}{n} \sum_{i=1}^{n} f_\theta\left(\textbf{s}_{i,c}\right),\\
\end{equation}
where $\bm{q}_l$ is the $l$th sample in query set, and  $\bm{s}_{i,c}$ is the $i$th sample in the $c$th class of the support set. In addition, $f_\theta(\cdot)$ with the parameter $\theta$ is a network.

\textbf{Subspace type.} In addition to the point type, a class can also be represented as a subspace, e. g., the subspace spanned by the the feature vectors created from the support samples in a class \cite{simon2020adaptive}. In this case, a sample should be close to its own class subspace and stay away from the subspaces of other classes. For example, for DSN, the distance square is expressed as

\begin{equation}\label{DSN_m}
d_{l,c}^2 = \Vert (\bm{I} - \bm{P}_{c}\bm{P}_{c}^T)(f_\theta\left(\bm{q}_l\right) - \bm{\mu}_{c})\Vert^2,\,\,\, {\rm{with}}\,\,\,\bm{\mu}_{c} = \frac{1}{n} \sum_{i=1}^{n} f_\theta\left(\textbf{s}_{i,c}\right),\\
\end{equation}
where $P_c$ the truncated matrix of $W_c$ the eigenvector matrix of empirical covariance matrix of support set. For a SVM classifier in few-shot learning, the representation of each class is one or more subspaces whose boundaries are so-called hyperplanes, which are determined by the support vectors \cite{lee2019meta-learning}. In this case, a sample should belong to its own class subspace and stay away from the hyperplanes.

In the following, we show that the class representations of the two types are not entirely different. For example, we demonstrate that our framework can be embodied as either of ProtoNet and DSN with different filter functions (see the Section 3.4).

\subsection{Kernel Shrinkage Relative-prototype}
Here we propose the definition of kernel shrinkage relative-prototype. Given the $C$-way $n$-shot support set $\bm{S}=\{\bm{S}_{1},\bm{S}_{2},...,\bm{S}_{C}\}$ with $\bm{S}_c=\{\bm{s}_{1,c},\bm{s}_{2,c},...,\bm{s}_{n,c}\}$, and the query set $\bm{Q}=\{\bm{q}_1,\bm{q}_2,...,\bm{q}_m\}$. We firstly use function $\phi_{\omega}$ to map the observations of support samples of class $c$ in feature space, and calculate the prototype as:

\begin{equation}\label{RProto}
\begin{split}
\bm{\mu}_{c} = \frac{1}{n} \sum_{i=1}^{n} \phi_{\omega}\left(f_\theta\left(\textbf{s}_{i,c}\right)\right),
\end{split}
\end{equation}
where $\phi_{\omega}(\cdot)$ with the parameter $\omega$ is a mapping function, and $f_\theta(\cdot)$ with the parameter $\theta$ is a network. Similar to the concepts in \cite{ye2020few,kang2021relational}, we propose the relative-prototype, the difference between a sample in the query set and a prototype of class as a mean form:

\begin{equation}\label{hatss}
\begin{split}
\bm{\mu}_{l,c} &= \phi_{\omega}\left(f_\theta\left(\bm{q}_l\right)\right) - \bm{\mu}_{c} = \frac{1}{n} \sum_{i=1}^{n}\bm{m}_{l,i,c},
\end{split}
\end{equation}
where

\begin{equation}\label{hatqs}
\bm{m}_{l,i,c} = \phi_{\omega} \left(f_\theta\left(\bm{q}_l\right)\right) - \phi_{\omega}\left(f_\theta\left(\bm{s}_{i,c}\right)\right).
\end{equation}
Simon et al. proposed a method that using adaptive subspaces instead of prototypes for few-shot learning \cite{simon2020adaptive}. In their work, the sample similarity is measured via a distance between a query sample and a subspace created from a support set, which can be seen as a ``shrinkage'' of the distance between a query sample and a prototype. Enlightened by it, we apply the shrinkage estimation theory to extend their idea and measure the sample similarity in a RKHS, or express the kernel shrinkage estimation of $\bm{\mu}_{l,c}$ as 

\begin{equation}\label{SRProto}
\begin{split}
\bm{\mu}_{l,c}(\lambda_c) = \bm{\mu}_{l,c} - \sum_{i=1}^{n} h(\gamma_{i,c}, \lambda_{l,c})\gamma_{i,c} \langle \bm{\mu}_{l,c}, \bm{w}_{i,c} \rangle \bm{w}_{i,c}
\end{split},
\end{equation}
where  $\lambda_{l,c}$ is the shrinkage coefficient with the class $c$, $(\gamma_{i,c}, \bm{w}_{i,c})$ are respectively the eigenvalues and eigenvectors of the empirical covariance matrix $C = \sum_{i=1}^{n}\bm{r}_{i,c} \otimes \bm{r}_{i,c}$ where 

\begin{equation}\label{r}
\bm{r}_{i,c} = \phi_{\omega}\left(f_\theta\left(\bm{s}_{i,c}\right)\right) - \bm{\mu}_{c}.
\end{equation}
In Eq. \ref{SRProto}, $h(\gamma_{i,c}, \lambda_{l,c})$ is a shrinkage function that approaches $1 / \gamma_{i,c}$ as $\lambda_{l,c}$ decreases to 0, implying that the relative prototype is fully shrunk;  As $\lambda_{l,c}$ increases, the shrinkage of relative prototype decreases or remain the same. There exist different ways of kernel mean shrinkage estimation that can be realized by constructing  $h(\gamma_{i,c}, \lambda_{l,c})$ differently, such as Tikhonov regularization and Truncated SVD \cite{muandet2014kernel}. In this work we apply the Tikhonov regularization as the filter function that

\begin{equation}\label{Tikhonov}
\begin{split}
h(\gamma_{i,c}, \lambda_{l,c}) = \frac{1}{\gamma_{i,c} + \lambda_{l,c}}.
\end{split}
\end{equation}

In other ways, however, the kernel relative-prototype shrinkage estimation $\bm{\mu}_{l,c}(\lambda_{l,c})$ in Eq. \ref{SRProto} is not a computable form as the mapping function $\phi_{\omega}$ is not or in some cases can not be known, e. g., $\phi_{\omega}: \mathbb{R}^d \to \mathbb{R}^\infty$. Thus, we propose a computable form of $\bm{\mu}_{l,c}(\lambda_{l,c})$ based on the work by Muandet et al. \cite{muandet2014kernel}. 

\newtheorem{myDef}{Definition}
\newtheorem{myTheo}{Theorem}

\begin{myTheo}\label{The2}
	Denote the $n\times n$ metrix $\bm{K}_{ss}^c$ whose entry at the row $i$ and the column $j$ ($\forall i, j$)  can be expressed the kernel form $k(f_\theta(\bm{s}_{i,c}),f_\theta(\bm{s}_{j,c})) = \phi_{\omega}(f_\theta(\bm{s_{i,c}}))^T\phi_{\omega}(f_\theta(\bm{s_{i,c}}))$, and the $n\times n$ metrix $\bm{K}_{qs}^{l,c}$ whose entry at the row $i$ and the column $j$ ($\forall i, j$) can be expressed the kernel form $k(f_\theta(\bm{s}_{i,c}),f_\theta(\bm{q}_{l,c}))= \phi_{\omega}(f_\theta(\bm{s}_{i,c}))^T\phi_{\omega}(f_\theta(\bm{q}_{l,c}))$. The kernel kean shrinkage estimation of $\bm{\mu}_{l,c}$ in Eq. \ref{SRProto} can be expressed as:
	\begin{equation}\label{SRProto_e}
	\begin{split}
	\bm{\mu}_{l,c}(\lambda_{l,c}) = \bm{\mu}_{l,c} - \sum_{i=1}^{n}\alpha_{l,i,c}(\lambda_{l,c})\bm{r}_{i,c},
	\end{split}
	\end{equation}
	with
	\begin{equation}\label{alpha}
	\begin{split}
	\bm{\alpha}_{l,c}(\lambda_{l,c})= g^h(\bm{\tilde{K}}_{ss}^c, \lambda_{l,c})\bm{\tilde{K}}_{qs}^{l,c} \bm{I}_{n},
	\end{split}
	\end{equation}
	\begin{equation}\label{s_ss}
	\bm{\tilde{K}}_{ss}^c = \bm{K}_{ss}^c - \bm{\hat{I}}_{n}\bm{K}_{ss}^c - \bm{K}_{ss}^c\bm{\hat{I}}_{n} + \bm{\hat{I}}_{n}\bm{K}_{ss}^c\bm{\hat{I}}_{n}, 
	\end{equation}
	\begin{equation}\label{s_qs}
	\bm{\tilde{K}}_{qs}^{l,c} = \bm{K}_{qs}^{l,c} - \bm{\hat{I}}_{n}\bm{K}_{qs}^{l,c} - \bm{K}_{ss}^c + \bm{\hat{I}}_{n}\bm{K}_{ss}^c,
	\end{equation}
	where $\bm{\alpha}_{l,c}(\lambda_{l,c}) = [\alpha_{l,1,c}(\lambda_{l,c}),...,\alpha_{l,n,c}(\lambda_{l,c})]^T$ and $\bm{I}_{n} = [1/n,1/n,...,1/n]^T$, and $\{\bm{\hat{I}}_{n}\}_{i,j} = 1/n$. Suppose the eigen-decomposition that $\bm{\tilde{K}}_{ss}^c = \bm{V \Psi V^T}$, then
	
	\begin{equation}\label{RK}
    	g\bm{(\tilde{K}}_{ss}^{c}, \lambda_{l,c}) =  \bm{V} g^h\left(\bm{\Psi},\lambda_{l,c}\right) \bm{V^T},
    \end{equation}
where $g(\bm{\Psi},\lambda_{l,c}) = {\rm diag}(h(\gamma_{1,c}, \lambda_{l,c}),...,h(\gamma_{n,c}, \lambda_{l,c}))$ with the $\bm{\tilde{K}}_{ss}^{c}$'s eigenvalues $\gamma_1,\gamma_2,...,\gamma_n$.
\end{myTheo}

\begin{proof}\label{pf2}
	Suppose that $v_{i,j}$ is the entry at the row $i$ and column $j$ of $\bm{V}$. According to \cite{scholkopf1998nonlinear}, we have $\textbf{w}_{i,c} = (1/\sqrt{\gamma_i})\sum_j^{n}v_{i,j}\bm{r}_{j,c}$,
	
	\begin{equation}\label{proof1}
		\begin{split}
			\bm{\mu}_{l,c}(\lambda_{l,c}) &= \bm{\mu}_{l,c} - \sum_{i=1}^{n} h(\gamma_{i,c}, \lambda_{l,c})\gamma_{i,c} \langle \bm{\mu}_{l,c}, \textbf{w}_{i,c} \rangle \textbf{w}_{i,c}\\
			&= \bm{\mu}_{l,c} -\sum_{j=1}^{n}\sum_{i=1}^{n} v_{i,j}h(\gamma_{i,c}, \lambda_{l,c}) \langle \bm{\mu}_{l,c}, \sum_k^{n}v_{i,k}\bm{r}_{k,c} \rangle \bm{r}_{j,c}\\
			&= \bm{\mu}_{l,c} - \sum_{j=1}^{n} \alpha_{l,j,c}(\lambda_{l,c}) \bm{r}_{j,c}.\\
		\end{split}
	\end{equation}
	where
	
	\begin{equation}\label{proof2}
		\begin{split}
			\alpha_{l,j,c}(\lambda_{l,c}) &= \sum_{i=1}^{n} v_{i,j}h(\gamma_{i,c}, \lambda_{l,c})\langle \bm{\mu}_{l,c}, \sum_{k=1}^{n}v_{i,k}\bm{r}_{k,c} \rangle\\
			& = \sum_{i=1}^{n} v_{i,j}h(\gamma_{i,c}, \lambda_{l,c}) \sum_k^{n}v_{i,k}\langle \bm{\mu}_{l,c}, \bm{r}_{k,c}. \rangle\\
		\end{split}
	\end{equation}
	Suppose that $\bm{v_{i}} = [v_{i,1}, v_{i,2},...,v_{i,n}]^T$,
	
	\begin{equation}\label{logits}
		\begin{split}
			\bm{\alpha}_{l,c}(\lambda_{l,c}) & = \sum_{i=1}^{n} \bm{v_{i}}h(\gamma_{i,c}, \lambda_{l,c})\sum_k^{n}v_{i,k}\langle \bm{\mu}_{l,c}, \bm{r}_{k,c} \rangle\\
			& = \sum_{i=1}^{n} \bm{v_{i}}h(\gamma_{i,c}, \lambda_{l,c})\bm{v_{i}}^T\bm{\tilde{K}}_{qs}^{l,c}\bm{I}_{n}\\
			& = g(\bm{\tilde{K}}_{ss}^c)\bm{\tilde{K}}_{qs}^{l,c} \bm{I}_{n} .
		\end{split}
	\end{equation}
	
\end{proof}

Based on Theorem \ref{The2}, we can calculate the similarity between each sample in query set and prototype of each class using $\bm{\mu}_{l,c}(\lambda_{l,c})$. 

\subsection{Shrinkage Base Classifiers}

For the convenience of calculation, we suppose that all the shrinkage parameters are the same, or $\lambda_{l,c} = \lambda$. The similarity of sample-pairs can be measured using the square of distance between the relative-prototype and original point in RKHS

\begin{equation}\label{dist}
\begin{split}
d_{l,c}^2 \left(\lambda,\bm{S}_{c}, \bm{q}_{l}\right)&= \displaystyle\left\Vert\bm{\mu}_{l,c}\left(\lambda\right)\right\Vert^2  \\      
&\displaystyle =\left(\bm{\alpha}_{l,c}\left(\lambda\right)\right)^T \bm{\tilde{K}}_{ss}^{c}       \bm{\alpha}_{l,c}\left(\lambda\right) + \bm{I}_{n}^T \bm{\tilde{K}}_{qq}^{l,c}\bm{I}_{n} - 2\left(\bm{\alpha}_{l,c}\left(\lambda\right)\right)^T \bm{\tilde{K}}_{qs}^{l,c} \bm{I}_{n}
\end{split}
\end{equation}
where $\bm{\tilde{K}}_{qq}^{l,c}$ can be written as

\begin{equation}\label{s_qq}
\bm{\tilde{K}}_{qq}^{l,c} = \bm{K}_{qq}^{l,c}+\bm{K}_{ss}^{c} - \bm{K}_{qs}^{l,c} - \left(\bm{K}_{qs}^{l,c}\right)^T ,
\end{equation}
with the $n\times n$ metrix $\bm{K}_{qq}^{l,c}$ whose entry at the row $i$ and the column $j$ ($\forall i, j$) can be expressed the kernel form $k(f_\theta(\bm{q}_{l}),f_\theta(\bm{q}_{l})) = \phi_{\omega}(f_\theta(\bm{q}_{l}))^T\phi_{\omega}(f_\theta(\bm{q}_{l}))$.The probability of the sample $\bm{q}_l$ in query set belonging to class $c$ is

\begin{equation}\label{prob}
P_{\omega,\theta}\left(Y=c|\bm{q}_l\right) = \displaystyle \frac{\exp\left(-\zeta d_{l,c}^2\left(\lambda,\bm{S}_{c}, \bm{q}_{l}\right)\right)}{\sum_{c=1}^C\exp\left(-\zeta d_{l,c}^2\left(\lambda,\bm{S}_{c}, \bm{q}_{l}\right)\right)},
\end{equation}
where $\zeta$ is the metric scaling parameter. The loss function of DSFN is

\begin{equation}\label{Loss}
\begin{split}
\mathcal{L}(\omega,\theta) =& -\frac{1}{m} \sum_{l=1}^{m}\log P_{\omega,\theta}\left(y_l|\bm{q}_l\right) 
\end{split},
\end{equation}
where $y_l$ is the label of $\bm{q}_l$. The few-shot learning process with the proposed DSFN is shown as Algorithm 1.

\begin{table}[!htbp]
\label{Alg}
\centering
\begin{tabular}{l}
\toprule[1pt]
\textbf{Algorithm 1} Few-shot learning with the proposed DSFN \\
\hline
\textbf{Input}: Support set $\textbf{S}$ and query set $\textbf{Q}$, learning rate $\alpha$. \\
\textbf{Output}: $\theta$\\
1:\,\,\,  Initialize $\theta$ randomly;\\
2:\,\,\,  \textbf{for} $t=1$ to $T$ \textbf{do} \\
3:\,\,\,\,\,\,\,\,   Generate episode by randomly sampling $\textbf{S}^{(t)}$ from $\textbf{S}$ and $\textbf{Q}^{(t)}$ from $\textbf{Q}$;\\
4:\,\,\,\,\,\,\,\,  \textbf{for} $c=1$ to $C$ \textbf{do} \\
5:\,\,\,\,\,\,\,\,\,\,\,\,\,\,\,\, \textbf{for} $l=1$ to $m$ \textbf{do} \\
6:\,\,\,\,\,\,\,\,\,\,\,\,\,\,\,\,\,\,\,\,\,\,\,  Compute $\bm{\tilde{K}}_{ss}^{c}$, $\bm{\tilde{K}}_{qs}^{l,c}$ and $\bm{\tilde{K}}_{qq}^{l,c}$ using Eq. \ref{s_ss}, Eq. \ref{s_qs} and Eq. \ref{s_qq},\\ 
\,\,\,\,\,\,\,\,\,\,\,\,\,\,\,\,\,\,\,\,\,\,\,\,\,\,\,  respectively, where $\bm{K}_{ss}^{c}$, $\bm{K}_{qs}^{l,c}$ and $\bm{K}_{qq}^{l,c}$ are calculated using the samples \\
\,\,\,\,\,\,\,\,\,\,\,\,\,\,\,\,\,\,\,\,\,\,\,\,\,\,\, in $\textbf{S}^{(t)}$ and $\textbf{Q}^{(t)}$;\\
7: \,\,\,\,\,\,\,  \,\,\,\,\,\,\,\,\,\,\, Compute  $\bm{\alpha}_{l,c}(\lambda_{l,c})$ with Eq. \ref{alpha},using $\bm{\tilde{K}}_{qs}^{l,c}$ and eigenvalue decomposition \\
\,\,\,\,\,\,\,\,\,\,\,\,\,\,\,\,\,\,\,\,\,\,\,\,\,\,\, of $\bm{\tilde{K}}_{ss}^{c}$;\\
8: \,\,\,\,\,\,\,  \,\,\,\,\,\,\,\,\,\,\, Compute  $d_{l,c}\left(\lambda^{(t)},\bm{S}^{(t)}_{c}, \bm{q}_{l}^{(t)}\right)$ using Eq. \ref{dist};\\
9:\,\,\,  \,\,\,\, \,\,\,\,\,\,\, \textbf{end for}\\
10:\,\,\,  \,\,\,  \textbf{end for}\\
11:\,\,\,  \,\,\,  Compute the loss function using Eq. \ref{Loss}; \\
12:\,\,\,  \,\,\, Update $\theta$ with $\theta - \alpha\nabla_\theta \mathcal{L} (\omega,\theta)$.\\
13:\,\,\,  \textbf{end for}\\

\bottomrule[1pt]
\end{tabular}
\end{table}

\subsection{Relationship to Other Methods}
Here we discuss the connection of our class representation to the point type (e. g., ProtoNet) and subspace type (e. g., DSN) representations. In fact, they are different mainly because they use different filter functions.

\textbf{Relationship to ProtoNet.} As the filter function $h(\gamma_{i,c}, \lambda_{l,c}) = 0$ instead of Tikhonov regularization that causes the disappearance of shrinkage effect, and the map function $\phi_{\omega}$ is identical, the proposed framework (Eq. \ref{SRProto}) is embodied as ProtoNet.

\textbf{Relationship to DSN.} While using the Truncated SVD as the filter function that $h(\gamma_{i,c}, \lambda_{l,c}) = \mathbb{I}_{(\gamma_{i,c} \geq \lambda_{l,c})}\gamma_{i,c}^{-1}$ instead of Tikhonov regularization, where $\mathbb{I}_{(\gamma_{i,c} \geq \lambda_{l,c})}\gamma_{i,c}^{-1}$ is the indicative function that is 1 if $\gamma_{i,c} \geq \lambda_{l,c}$ else 0, and the map function $\phi_{\omega}$ is identical, the proposed framework (Eq. \ref{SRProto}) is embodied as DSN. In this case, by setting different values of $\lambda_{l,c}$,  different dimension of subspace can be selected. Formally, the relationship of DSFN and DSN is shown in Theorem 2 (see detailed proof in supplementary material).

\begin{myTheo}\label{The1}
Suppose that: 1) $h(\gamma_{i,c}, \lambda_{l,c}) = \mathbb{I}_{(\gamma_{i,c} \geq \lambda_{l,c})}\gamma_{i,c}^{-1}$; 2) $\zeta = 1$; 3) $\lambda_{l,c} = constant$ for all $l$ and $c$; 4) the map function $\phi_{\omega}$ is identical. Eq. \ref{dist} is reduced to $d_{l,c}^2 \left(\lambda,\bm{S}_{c}, \bm{q}_{l}\right) = \Vert(I - P_cP_c^T)(f_\theta\left(\bm{q}_l\right) - \bm{\mu}_{c})\Vert^2$ with $P_c$ the truncated
matrix of $W_c$, where $W_c$ is the eigenvector matrix of empirical covariance matrix $C$. 	
\end{myTheo}

Theorem \ref{The1} implies that, while using Truncated SVD as the filter function, the loss function of our proposed framework (Eq. \ref{prob}) can be reduced to the loss of DSN with no regularization (See Eq. 5 in the work by Simon et al.\cite{simon2020adaptive}).

\section{Experiments Setup}
\subsection{Datasets}
\textbf{\emph{mini}ImageNet}. \emph{mini}ImageNet dataset \cite{vinyals2016matching} was often used for few-shot learning, which contains a total of 60,000 color images in 100 classes randomly selected from ILSVRC-2012, with 600 samples in each class. The size of each image is 84 $\times$ 84. In the data set, the training set, validation set and test set contains the number of classes with 64 : 16 : 20. 

\textbf{\emph{tiered}-ImageNet}. The \emph{tiered}-ImageNet dataset \cite{ren2018meta} is a benchmark image dataset that is also selected from ILSVRC-2012 but contains 608 classes that is more than that in \emph{mini}ImageNet dataset. These classes are divided into 34 high-level categories, can each category contains 10 to 30 classes.  The size of each image is 84 $\times$ 84. Further, the categories are divided into the training set, validation set and test set with 20 : 6 : 8.

\textbf{CIFAR-FS}. The CIFAR-FS dataset \cite{bertinetto2018meta} is a few-shot learning benchmark containing all 100 classes from CIFAR-100 \cite{krizhevsky2009learning}, and each class contains 600 samples. The size of each image is 32 $\times$ 32. The classes are divided into the training set, validation set and test set with 64 : 16 : 20. 


\subsection{Implementation}
In training stage, 15-shot 10-query samples are chosen on \textit{mini}ImageNet dataset; 10-shot 15-query samples are chosen on \textit{tiered}ImageNet dataset;  2-shot 20-query samples are chosen for 1-shot task and 15-shot 10-query samples are chosen for 5-shot task on  CIFAR-FS dataset. $\lambda$ is set to the best by choosing 0.01,0.1,1,10 or 100. For these datasets, the setting of 8 episodes per batch is utilized in the experiments. The total number of training epochs is 80, and in each epoch 1000 batches are sampled. In testing stage,  1000 episodes are used to assess our model. For the 1-shot K-way learning, another support sample was created by flipping the original support sample, and two support samples are used for spectral filtering in the validation and testing stages. Our model is trained and tested in the PyTorch machine learning package \cite{paszke2017automatic}.

Two backbones, the Conv-4 and Resnet-12, are utilized as the backbones in our model. For Conv-4, the Adam optimizer with default momentum values ([$\beta_1$, $\beta_2$] = [0.9, 0.999]) is applied for the training. The learning rate is initially set as 0.005 then decayed to 0.0025, 0.00125,  0.0005 and 0.00025 at 8, 30, 45 and 50 epochs, respectively. For ResNet-12,  the SGD optimizer is applied for the training, and the learning rate is initially set as 0.1 then decayed to 0.0025, 0.00032, 0.00014 and 0.000052 at 12, 30, 45 and 57 epochs, respectively.

The identity kernel $k^{ide}(\bm{z}_i,\bm{z}_j) = \langle \bm{z}_i, \bm{z}_j\rangle$ and the RBF kernel $k^{rbf}(\bm{z}_i,\bm{z}_j) = \exp \left( - \| \bm{z}_i - \bm{z}_j \|^2/ (2\sigma^2) \right)$ are chosen as the kernel functions, where $\sigma^2$ is assigned as the dimension of embeddings. In addition, the scaling parameter $\zeta$ is learned as a variable. For the filter function, we set the shrinkage coefficient $\lambda$ as the fixed multiple of the maximum eigenvalue with each class.

\section{Experiments and Discussions}

\subsection{Comparison with State-of-the-art Methods}

\textbf{Results on \textit{mini}ImageNet dataset.} Firstly, the proposed DSFN and the state-of-the art methods for 5-way classification tasks on \textit{mini}ImageNet dataset are compared in Table \ref{mini}. Table \ref{mini} shows that the proposed DSFN with identity kernel can achieve the bests on 5-way 5-shot classification tasks using both Conv-4 and ResNet-12 backbones, which are higher than DSN with 3.3$\%$ and $1.3\%$, respectively. The RBF kernel can achieve the second best on 5-way 5-shot classification tasks.  These results illustrate that the proposed DSFN can achieve the state-of-the-art performance for 5-way 5-shot classification tasks on the dataset.

\textbf{Results on \textit{tiered}-ImageNet dataset.}   We compare the proposed DSFN with the state-of-the-art methods for 5-way classification tasks on \textit{tiered}ImageNet dataset, as shown in Table \ref{Tiered}. It can be seen that, the performance of DSFN with identity kernel and RBF kernel is slightly lower than DSN on 5-way 5-shot classification task. However, they are better than the others on 5-way 5-shot classification task. 

\textbf{Results on CIFAR-FS dataset.} A further comparison is made on CIFAR-FS dataset, as shown in Table \ref{CIFAR-FS}. Table \ref{CIFAR-FS} shows that the proposed DSFN with RBF kernel performs the best on 5-way 5-shot classification task, whose test accuracy is about 1.2$\%$ and 2.8$\%$ higher than those of DSN and ProtoNet, respectively. Thus, the proposed DSFN performs the state-of-the-art for 5-way 5-shot classification task on the dataset.

\begin{table}[h]

\centering
\begin{threeparttable}
\caption{Test accuracies ($\%$) from the proposed DSFN and the state-of-the art methods for 5-way  tasks on \textit{mini}ImageNet dataset with 95$\%$ confidence intervals. $\ddag$ means that training set and validation set are used for training the corresponding model.}
\begin{tabular}{ccccc}
\toprule[1pt]
\textbf{} 
\multirow{2}{*}{\bf{Model}} & \,\,\,\,\,\,\,\,\,\,\,\,\multirow{2}{*}{\bf{Backbone}}\,\,\,\,\,\,\,\,\,\,\,\, &\multicolumn{2}{c}{\bf{5-way}}\\
\cmidrule(r){3-4}
& &\bf{1-shot} &\,\,\,\,\,\,\,\,\,\,\,\,\,\,\,\,\bf{5-shot}\,\,\,\,\,\,\,\,\,\,\,\,\,\,\,\, \\
\hline
MatchingNet\cite{vinyals2016matching}         &Conv-4         &$43.56 \pm 0.84$   &$55.31 \pm 0.73$   \\			
MAML \cite{finn2017model-agnostic}            &Conv-4         &$48.70 \pm 1.84$   &$63.11 \pm 0.92$   \\
Reptile \cite{nichol2018first}                &Conv-4         &$49.97 \pm 0.32$   &$65.99 \pm 0.58$   \\
ProtoNet\cite{snell2017prototypical}          &Conv-4         &$44.53 \pm 0.76$   &$65.77 \pm 0.66$   \\
Relation Nets \cite{2018Learning}             &Conv-4         &$50.44 \pm 0.82$   &$65.32 \pm 0.70$   \\
DSN\cite{simon2020adaptive}                   &Conv-4         &$51.78 \pm 0.96$   &$68.99 \pm 0.69$   \\
\hline	
$\textbf{DSFN(identity kernel)}$             &Conv-4          &$\bf{50.21 \pm 0.64}$    &$\bf{72.20 \pm 0.51 }$    \\
$\textbf{DSFN(RBF kernel)}$                  &Conv-4          &$\bf{49.97 \pm 0.63}$    &$\bf{72.04 \pm 0.51}$    \\
\hline 
SNAIL\cite{mishra2018a}                       &ResNet-12         &$55.71 \pm 0.99$   &$68.88 \pm 0.92$   \\  
TADAM\cite{oreshkin2018tadam:}                &ResNet-12         &$58.50 \pm 0.30$   &$76.70 \pm 0.30$   \\
AdaResNet\cite{munkhdalai2018rapid}           &ResNet-12         &$56.88 \pm 0.62$   &$71.94 \pm 0.57$   \\
LEO$\ddag$\cite{rusu2018meta}                 &WRN-28-10         &$61.76 \pm 0.08$   &$77.59 \pm 0.12$   \\
LwoF\cite{gidaris2018dynamic}                 &WRN-28-10         &$60.06 \pm 0.14$   &$76.39 \pm 0.11$   \\
wDAE-GNN$\ddag$\cite{gidaris2019generating}   &WRN-28-10         &$62.96 \pm 0.15$   &$78.85 \pm 0.10$   \\
MetaOptNet-SVM\cite{lee2019meta-learning}      &ResNet-12     &$62.64 \pm 0.61$   &$78.63 \pm 0.46$   \\	
DSN\cite{simon2020adaptive}                   &ResNet-12         &$62.64 \pm 0.66$   &$78.83 \pm 0.45$   \\
CTM\cite{li2019finding}                       &ResNet-18         &$62.05 \pm 0.55$   &$78.63 \pm 0.06$   \\
Baseline\cite{chen2019closer}                 &ResNet-18         &$51.75 \pm 0.80$   &$74.27 \pm 0.63$   \\
Baseline++\cite{chen2019closer}               &ResNet-18         &$51.87 \pm 0.77$   &$75.68 \pm 0.63$   \\
Hyper ProtoNet\cite{khrulkov2020hyperbolic}   &ResNet-18         &$59.47 \pm 0.20$   &$76.84 \pm 0.14$   \\
Hyperbolic RBF kernel\cite{fang2021kernel}    &ResNet-18         &$60.91 \pm 0.21$   &$77.12 \pm 0.15$   \\
\hline
$\textbf{DSFN(identity kernel)}$             &ResNet-12         &$\bf{61.27 \pm 0.71}$   &$\bf{80.13 \pm 0.17}$   \\
$\textbf{DSFN(RBF kernel)}$                  &ResNet-12         &$\bf{59.43 \pm 0.66}$   &$\bf{79.60 \pm 0.46}$  \\
\bottomrule[1pt]
\end{tabular}\label{mini}
\end{threeparttable}
\end{table}

\begin{table}[H]
\centering
\begin{threeparttable}
\caption{Test accuracies ($\%$) from the proposed DSFN and the state-of-the art methods for 5-way tasks on \textit{tiered}-ImageNet dataset with 95$\%$ confidence intervals. $\ddag$ means that training set and validation set are used for training the corresponding model. }
\begin{tabular}{ccccc}
\toprule[1pt]
\textbf{} 
\multirow{2}{*}{\bf{Model}} &\,\,\,\,\,\,\,\,\,\,\,\, \multirow{2}{*}{\bf{Backbone}}\,\,\,\,\,\,\,\,\,\,\,\, &\multicolumn{2}{c}{\bf{5-way}}\\
\cmidrule(r){3-4}
& &\bf{1-shot}  &\,\,\,\,\,\,\,\,\,\,\,\,\,\,\,\,\bf{5-shot}\,\,\,\,\,\,\,\,\,\,\,\,\,\,\,\, \\
\hline
ProtoNet\cite{snell2017prototypical}           &ResNet-12         &$61.74 \pm 0.77$   &$80.00 \pm 0.55$   \\
CTM\cite{li2019finding}                        &ResNet-18         &$64.78 \pm 0.11$   &$81.05 \pm 0.52$   \\
LEO$\ddag$\cite{rusu2018meta}                  &WRN-28-10         &$66.33 \pm 0.05$   &$81.44 \pm 0.09$   \\
MetaOptNet-RR\cite{lee2019meta-learning}      &ResNet-12         &$65.36 \pm 0.71$   &$81.34 \pm 0.52$   \\
MetaOptNet-SVM\cite{lee2019meta-learning}      &ResNet-12         &$65.99 \pm 0.72$   &$81.56 \pm 0.53$   \\
DSN \cite{simon2020adaptive}                   &ResNet-12         &$66.22 \pm 0.75$   &$82.79 \pm 0.48$   \\
\hline
$\textbf{DSFN(identity kernel)}$               &ResNet-12         &$\bf{65.46 \pm 0.70}$   &$\bf{82.41 \pm 0.53}$   \\
$\textbf{DSFN(RBF kernel)}$                    &ResNet-12         &$\bf{64.27 \pm 0.70}$   &$\bf{82.26 \pm 0.52}$  \\
\bottomrule[1pt]
\end{tabular}\label{Tiered}
\end{threeparttable}
\end{table} 

\begin{table}[h]
\centering
\begin{threeparttable}
\caption{Accuracies ($\%$) from the proposed DSFN and some state-of-the art methods for 5-way classification tasks on CIFAR-FS dataset with 95$\%$ confidence intervals. }
\begin{tabular}{ccccc}
\toprule[1pt]
\textbf{} 
\multirow{2}{*}{\bf{Model}} &\,\,\,\,\,\,\,\,\,\,\,\, \multirow{2}{*}{\bf{Backbone}}\,\,\,\,\,\,\,\,\,\,\,\, &\multicolumn{2}{c}{\bf{5-way}}\\
\cmidrule(r){3-4}
& &\bf{1-shot}  &\,\,\,\,\,\,\,\,\,\,\,\,\,\,\,\,\bf{5-shot}\,\,\,\,\,\,\,\,\,\,\,\,\,\,\,\, \\
\hline
ProtoNet\cite{snell2017prototypical}          &ResNet-12       &$72.2 \pm 0.7$      &$83.5 \pm 0.5$  \\
MetaOptNet-RR\cite{lee2019meta-learning}      &ResNet-12       &$72.6 \pm 0.7$      &$84.3 \pm 0.5$   \\
MetaOptNet-SVM\cite{lee2019meta-learning}     &ResNet-12       &$72.0 \pm 0.7$      &$84.2 \pm 0.5$   \\	
DSN\cite{simon2020adaptive}                   &ResNet-12       &$72.3 \pm 0.7$      &$85.1 \pm 0.5$\\
\hline
$\textbf{DSFN(identity kernel)}$              &ResNet-12       &$\bf{70.62 \pm 0.79}$ &$\bf{86.11 \pm 0.58}$   \\
$\textbf{DSFN(RBF kernel)}$                   &ResNet-12       &$\bf{71.28 \pm 0.70}$ &$\bf{86.30 \pm 0.58}$  \\
\bottomrule[1pt] 
\end{tabular}\label{CIFAR-FS}
\end{threeparttable} 
\end{table}
\subsection{Ablation Study}
\textbf{The impact of shrinkage parameter.} The influence of different values of shrinkage parameter $\lambda$ on the performances of the proposed DSFN is shown in Fig. \ref{Ablation0}. Fig. \ref{Ablation0} shows that the general trends drop for these datasets as the shrinkage parameter increases from 0.01 to 100, and the descending trends with 5 shot is more obvious that those with 1 shot. These results indicate that smaller shrinkage parameters (e.g., 1, 0.1, 0.01) or stronger shrinkage effect can better improve the performance of the proposed model.

\begin{figure}[h]
	\centering
	{\includegraphics[width=0.85\textwidth]{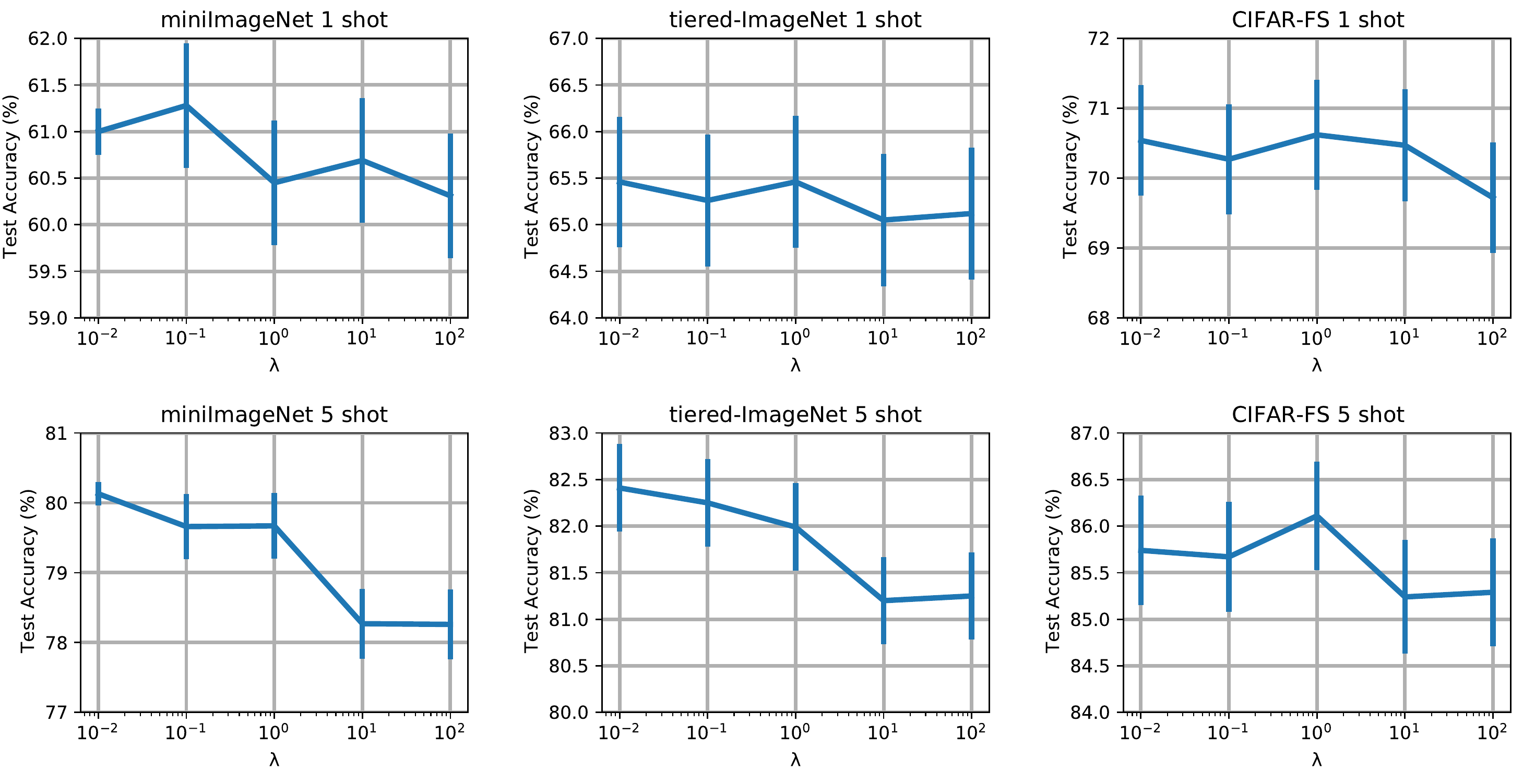}}
	\caption[]{Test accuracies on few-shot classification tasks from the proposed DSFN against different values of shrinkage parameter, where ResNet-12 is used. }
	\label{Ablation0}
\end{figure}
\renewcommand{\arraystretch}{1}

\setlength\tabcolsep{3pt}
\begin{table}[h]
\centering
\begin{threeparttable}
\caption{Accuracies ($\%$) from models with and without shrinkge for 5-way 1-shot and 5-way 5-shot classification tasks, w S: with shrinkge, w/o S: without shrinkge.}
\begin{tabular}{cccccc}
\toprule[1pt]
\textbf{} 
\bf{Dataset}&\,\,\,\,\,\,\,\,\bf{Kernel}\,\,\,\,\,\,\,\,    &\bf{w S} &\,\,\,\,\,\,\,\,\bf{w/o S}\,\,\,\,\,\,\,\, &{\bf{1-shot}}&\,\,\,\,\,\,\,\,{\bf{5-shot}}\,\,\,\,\,\,\,\, \\
\hline
\multirow{4}{*}{\bf{\textit{mini}ImageNet}}&identity&&\checkmark    &$60.14 \pm 0.67$      &$77.65 \pm 0.52$   \\
&identity&\checkmark&                                               &$61.00 \pm 0.25$      &$80.13 \pm 0.17$   \\
&RBF&&\checkmark                                                    &$58.74 \pm 0.65$      &$79.18 \pm 0.46$ \\
&RBF&\checkmark&                                                    &$59.43 \pm 0.66$      &$79.60 \pm 0.46$   \\
\hline
\multirow{4}{*}{\bf{\textit{tiered}-ImageNet}}&identity&&\checkmark  &$65.05 \pm 0.72$  &$81.14 \pm 0.55$   \\
&identity&\checkmark&                                                &$65.46 \pm 0.70$     &$82.41 \pm 0.53$   \\
&RBF&&\checkmark                                                     &$64.23 \pm 0.70$     &$82.07 \pm 0.53$   \\
&RBF&\checkmark&                                                     &$64.27 \pm 0.70$     &$82.26 \pm 0.52$   \\
\hline
\multirow{4}{*}{\bf{CIFAR-FS}}&identity&&\checkmark               &$70.54 \pm 0.82$     &$85.30 \pm 0.59$   \\
&identity&\checkmark&                                             &$70.62 \pm 0.79$     &$86.11 \pm 0.58$   \\
&RBF&&\checkmark                                                  &$71.18 \pm 0.73$     &$86.09 \pm 0.47$   \\
&RBF&\checkmark&                                                  &$71.28 \pm 0.70$     &$86.30 \pm 0.46$   \\
\bottomrule[1pt]
\end{tabular}\label{Ablation}
\end{threeparttable}
\end{table}

\textbf{The effectiveness of shrinkage.} We show an ablation study to illustrate the effectiveness of shrinkage in our work, as shown in Table \ref{Ablation}. In this experiment, the performances of identity kernel and RBF kernel with and without shrinkage in few-shot classification tasks are compared. Table \ref{Ablation} shows that, the proposed model with shrinkage performs better than those without shrinkage for both kernels. In addition, the improvement of shrinkage on 5-way 5-shot classification tasks is more obvious than those on 5-way 1-shot classification tasks, probably because the eigenvalues and eigenvectors with one shot is hard to learn. 

\subsection{Time Complexity}
Our proposed DSFN approach has the time complexity of $\mathcal{O}(\max(CN^3,CN^2D))$, where $C$, $N$, $D$ are the number of way, shot and feature dimensionality, respectively. The proposed DSFN approach is slower than DSN ($\mathcal{O}(\min(CND^2,CN^2D))$) and ProtoNet ($\mathcal{O}(CND)$) due to the kernel matrix calculation, eigen-decomposition and multiple matrix multiplication. A way to reduce the time complexity is using some efficient algorithms, such as the fast adaptive eigenvalue decomposition~\cite{chonavel2003fast} and faster matrix multiplication \cite{alman2021refined}.\\

\section{Conclusion}

In this work, we propose a framework called DSFN for few-shot learning. In this framework, one can represent the similarity between a query and a prototype as the distance after spectral filtering of support set in each class in RKHS. DSFN is an extension of some mainstream methods, e. g., ProtoNet and DSN, and with appropriate filter function, the framework of DSFN can be embodies as those methods. In addition, we also showed that in this framework, one can explore new methods by applying diverse forms, e. g., Tikhonov regularization, as the filter function, and diverse forms of kernels. Several experiments verified the effectiveness of various specific forms of the proposed DSFN. Future works should take a closer look at the selection of filter function and the role of shrinkage parameter in the proposed framework.


\bibliographystyle{splncs04}
\bibliography{egbib}
\end{document}